%% file: draft.tex
\documentclass[conference,compsoc]{IEEEtran}
\usepackage{url}
\usepackage{graphicx}
\usepackage{xcolor}
\usepackage{amsfonts}
\usepackage{amsmath}
\usepackage[ruled,vlined,linesnumbered]{algorithm2e}
\usepackage{mathtools}
\usepackage{amsthm}
\usepackage{caption}
\usepackage{subcaption}

\SetCommentSty{mycommfont}

\newcommand{\sig}[1]{{\small\textsf{{#1}}}}
\newcommand{\Comment}[1]{}

\newtheorem{proposal}{Safety Metric Proposal}
\newtheorem{lem}{Lemma}

\usepackage{tikz}
\usepackage{textcomp}
\usepackage{hyperref}
\usepackage{lipsum}
\newcommand\copyrighttext{%
  \footnotesize \textcopyright 2021 IEEE. Personal use of this material is permitted.
  Permission from IEEE must be obtained for all other uses, in any current or future
  media, including reprinting/republishing this material for advertising or promotional
  purposes, creating new collective works, for resale or redistribution to servers or
  lists, or reuse of any copyrighted component of this work in other works.
  DOI: 10.1109/AITEST52744.2021.00021}
\newcommand\copyrightnotice{%
\begin{tikzpicture}[remember picture,overlay]
\node[anchor=south,yshift=10pt] at (current page.south) {\fbox{\parbox{\dimexpr\textwidth-\fboxsep-\fboxrule\relax}{\copyrighttext}}};
\end{tikzpicture}%
}

\begin{document}
%
\title{Safety Metrics for Semantic Segmentation in Autonomous Driving}

\author{\IEEEauthorblockN{Chih-Hong Cheng}
\IEEEauthorblockA{DENSO AUTOMOTIVE Deutschland GmbH\\
Eching, Germany\\
Email: c.cheng@eu.denso.com}
\and
\IEEEauthorblockN{Alois Knoll}
\IEEEauthorblockA{TU Munich\\
Munich, Germany\\
Email: knoll@in.tum.de}
\and
\IEEEauthorblockN{Hsuan-Cheng Liao}
\IEEEauthorblockA{DENSO AUTOMOTIVE Deutschland GmbH\\
Eching, Germany\\
Email: h.liao@eu.denso.com}}

\maketitle
\copyrightnotice

\begin{abstract}
Within the context of autonomous driving, safety-related metrics for deep neural networks have been widely studied for image classification and object detection. In this paper, we further consider safety-aware correctness and robustness  metrics specialized for \emph{semantic segmentation}. The novelty of our proposal is to move beyond pixel-level metrics: Given two images with each having~$n$ pixels being class-flipped, the designed metrics should, depending on the \emph{clustering of pixels being class-flipped} or \emph{the location of occurrence}, reflect a different level of safety criticality. The result evaluated on an autonomous driving dataset demonstrates the validity and practicality of our proposed methodology. 
\end{abstract}


%
\IEEEpeerreviewmaketitle

\input{sections/1_Introduction}

\input{sections/2_RelatedWork}

\input{sections/3_SafetyAwareness}

\input{sections/4_Evaluation}

\input{sections/5_Conclusion}

\vspace{5mm}
\noindent\textbf{(Acknowledgement)} This project has received funding from the European Union’s Horizon 2020 research and innovation programme under grant agreement No 956123.

\bibliographystyle{abbrv}

\end{document}

%% file: sections/1_Introduction.tex
\section{Introduction}\label{sec.intro}

While deep neural networks (DNNs) have been widely employed in perception systems of autonomous vehicles, the safety problem caused by unreliable DNNs is one of the primary barriers in massively deploying the technology. In this paper, we consider \emph{safety-aware  metrics} for \emph{semantic segmentation} applications. Targeting semantic segmentation differentiates our work from prior results, as existing work either focuses on the overall machine learning paradigm or is restricted to \emph{image classification} or \emph{object detection} domains.

Given an image, the semantic segmentation DNN assigns each pixel one class within the set of the predefined classes. When designing safety-related metrics for semantic segmentation, an intuitive approach is to reuse classification-oriented safety metrics. However, it leads to counter-intuitive and hard-to-interpret associations with safety. This is due to the observation that in semantic segmentation, classification is operated on the pixel-level, and simply \emph{averaging} pixel-level metrics lacks information which allows for understanding the impact of \emph{clustering} and \emph{occurring location}. The key contribution of this paper is thus the proposal of new safety-aware metrics to be used for semantic segmentation. The metrics can be used as an evaluation function in testing, providing evidence that the prediction does not lead to safety-related concerns. During our investigation, we discovered intriguing properties regarding the non-monotonic change of error density subject to the enlargement of filter sizes. To avoid computing metrics for all possible filter sizes, we developed an efficient algorithm that iteratively refines the upper bound for the filter size to be applied. 

A preliminary evaluation, by implementing the mentioned metrics into TensorFlow 2.0 and evaluating over the Berkeley Cityscapes dataset \cite{Cordts2016Cityscapes}, a public semantic segmentation benchmark, clearly indicates the appropriateness of our proposals. Altogether our work provides an initial step towards safety argumentation for semantic segmentation neural networks: for semantic segmentation, with carefully crafted safety-aware metrics, it is possible to (without sacrificing safety) deploy a DNN that does not generate  perfectly accurate pixel-level prediction in training and testing.

The remainder of the paper is structured as follows. Section~\ref{sec.related.work} presents an overview of related work and outlines the difference of our approach. Section~\ref{sec.safety.aware.metrics} gives a brief summary of semantic segmentation and  presents the associated safety-aware metrics. Finally, we discuss our preliminary evaluation in Section~\ref{sec.evaluation} and conclude with future work in Section~\ref{sec.concluding.remarks}.

%% file: sections/2_RelatedWork.tex
\section{Related Work}\label{sec.related.work}

\subsection{Importance-Aware Semantic Segmentation}\label{subsec.importance.aware}

Semantic segmentation using DNN from established results such as FCN~\cite{long2015fully}, U-Net~\cite{ronneberger2015u}, SegNet~\cite{badrinarayanan2017segnet} and DeepLab~\cite{chen2017deeplab}
can be found in vision pipelines for free space detection, road or land boundary identification. These applications enables algorithmic diversity within the single sensor modality.  
We refer readers to the review~\cite{siam2018comparative} for a comparative study on evaluating semantic segmentation in autonomous driving. 

During training, semantic segmentation models normally apply the pixel-wise cross-entropy loss, viewing each output of a pixel as a classifier. To tackle class imbalance in semantic segmentation, e.g. there are generally more pixels of \sig{road} than other classes such as \sig{pedestrian}, researchers adapt additional loss function including weighted cross-entropy~\cite{paszke2016enet} and focal loss~\cite{lin2017focal}. Nonetheless, the introduction of these loss functions aims at improving the accuracy in the training process and they are not safety-aware. The recent work on \emph{importance-aware semantic segmentation}~\cite{liu2020importance} addresses the issue that certain types of misclassification are more important than others. For instance, an incorrect classification from a \sig{car} to a \sig{bus} should be less punished than one from a \sig{car} to the \sig{sky}. Consequently, the authors define a matrix, based on the Wasserstein distance, to associate each type of class flip with a different weight factor. This idea is a natural generalization of importance-aware object detection with examples such as weighted confusion, where the error for a \sig{pedestrian} object being identified as a \sig{background/no-object} should be smaller than the opposite, considering the safety-criticality. Our paper naturally covers the limitation of this line of work: apart from pixel-level information in ordinary semantic segmentation, our safety-aware metrics further characterize the clustering of misclassified pixels as well as their occurring locations in one image.

\subsection{Robustness and Safety Remedies}
\label{subsection:Robustness}

In object classification and detection, there are varied definitions of robustness (see~\cite{huang2020survey} for an overview), either emphasizing pixel-level perturbations using L-norms or ``natural transformations" such as weather changes. The definition of the former normally considers class-flips and would thus require proper adaptations (as suggested in this paper) for semantic segmentation.

Perturbation robustness, also known as adversarial robustness, generally refers to a model's ability to handle small changes within the input images that are able to deceive the underlying DNN and cause errorneous model behaviors. In comparison to the research efforts in perturbation robustness of image classifiers \cite{szegedy2014intriguing,goodfellow2015explaining,madry2019deep}, there have not been many focuses in semantic segmentation scenarios. Arnab \textit{et al.}~\cite{arnab2018robustness}, being among the first few, evaluate different models' robustness against FGSM-based attacks, including single-step/iterative and untargeted/targeted operations. They suggest that techniques such as residual connections, multiscale processing and Conditional Random Fields (CRFs) could lead to better model performance in terms of accuracy and robustness. More recently, Xu \textit{et al.} \cite{xu2020dynamic} devise a dynamic divide-and-conquer algorithm to robustly train their DNN, while Klinger \textit{et al.} \cite{klingner2020improved} employ an auxiliary depth-estimating branch for redundancy. 

Although the above studies have shown various extents of success, there is a common absence of an effective metric that reflects safety-criticality in different scenarios. For example, a perturbation-caused performance degradation in the \sig{sky} should be considered less relevant than one on a \sig{pedestrian}.
The lack of existing safety-aware metrics is likely due to the difficulties in defining the norm of safety and deriving it from DNN models. Only until recently that Xu \textit{et al.} \cite{xu2020quantification} attempt to quantify safety risks of DNNs in image classification tasks with respect to perturbation robustness. 
Specifically, the authors regard the discrepancy of two confidence values, e.g. the two largest softmax probabilities, as an robustness indicator.
Then, they propose a Lipschitzian metric to approximate this quantifier, and implement a software tool to evaluate DNN classifiers. Despite their paradigm, our work targets directly at semantic segmentation applications, in which the pixel-level classifications usually require a distinct treatment.

Finally, there are generic argumentation frameworks (e.g.,~\cite{UL4600,zhao2020safety,jia2021framework}) addressing potential insufficiencies in engineering machine learning components, including data collection, training, testing and runtime monitoring. The ISO Technical Report~4804~\cite{saFAD} has a dedicated appendix focusing on the safety engineering of DNNs, but again the detailed use case is for 3D object detection rather than semantic segmentation.

%% file: sections/3_SafetyAwareness.tex
\section{Safety-aware Performance Characterization for Semantic Segmentation} \label{sec.safety.aware.metrics}

\subsection{Foundation}

In this paper, we consider semantic segmentation over 2D images. Let $\mathcal{L}$ be the list of \emph{labels} with each element being disjoint from others, and $|\mathcal{L}|$ the size of the list. For example, for the Berkeley Cityscapes dataset~\cite{Cordts2016Cityscapes},  $\mathcal{L}$ has a size of~$30+$ elements. Given an input image $\sig{in}$ with number of pixels being $L\times W$, a \emph{semantic segmentation network} is a function $\sig{snn}$   that assigns, for each pixel indexed $(i,j) \in ([1, L] \cap \mathbb{Z}) \times ([1, W] \cap \mathbb{Z})$, a vector $\sig{snn}(\sig{in}, (i,j)) = (v_1, \ldots, v_{|\mathcal{L}|})$. Let $\sig{snn}_{k}(\sig{in}, (i,j))$ be the $k$-th value of the vector, then the semantic segmentation network predicts the pixel $(i,j)$ to be class~$k$ if for all other $k' \in [1, |\mathcal{L}|] \cap \mathbb{Z}$, $\sig{snn}_{k}(\sig{in}, (i,j)) \geq \sig{snn}_{k'}(\sig{in}, (i,j))$. We use the term $\sig{argmax}(\sig{snn}(\sig{in}, (i,j)))$ to denote the generated prediction.\footnote{Here the formulation is simplified in that it is possible to have two entries in $\textsf{snn}(\sig{in}, (i,j))$ with the same maximal value. We assume that the \textsf{argmax} operator always returns the smallest index to avoid output ambiguity. } 

Finally, the \emph{ground-truth} of an image $\sig{in}$ can also be viewed as function $\sig{gt}$   that assigns, for each pixel indexed $(i,j) \in ([1, L] \cap \mathbb{Z}) \times ([1, W] \cap \mathbb{Z})$, a value $\sig{gt}(\sig{in}, (i,j)) \in [1, |\mathcal{L}|] \cap \mathbb{Z}$. 

\subsection{Developing pixel-level correctness and robustness metrics}\label{subsec.pcm.prm}

With the above formulation, one can view each pixel as independent and consider the correctness and robustness metric on the pixel level. 
Given an image~$\sig{in}$, a semantic segmentation network $\sig{snn}$ generates correct prediction on the pixel indexed $(i,j)$ when the following condition holds.
\vspace{-2mm}
\begin{equation*}
\sig{argmax}(\sig{snn}(\sig{in}, (i,j))) = \sig{gt}(\sig{in}, (i,j)) 
\end{equation*}

Given an image~$\sig{in}$, a semantic segmentation network $\sig{snn}$ generates \emph{completely correct prediction} when all pixels are predicted correctly, that is, 
\vspace{-2mm}
\begin{multline*}
\forall (i, j) \in ([1, L] \cap \mathbb{Z}) \times ([1, W] \cap \mathbb{Z}):\\ \sig{argmax}(\sig{snn}(\sig{in}, (i,j))) = \sig{gt}(\sig{in}, (i,j))
\end{multline*}

We can define~$\sig{pcm}(\sig{snn}, \sig{in})$, the \emph{prediction correctness metric} of a semantic segmentation network $\sig{snn}$ under an input image $\sig{in}$, to be the ratio where  pixel are predicted correctly, i.e., 
\begin{multline*}\sig{pcm}(\sig{snn}, \sig{in}) := \\ \frac{|\{(i,j)\;|\;\sig{argmax}(\sig{snn}(\sig{in}, (i,j))) = \sig{gt}(\sig{in}, (i,j))\}|}{L\times W}
\end{multline*}

We can also describe the \emph{pixel-level robustness} for the semantic segmentation network over the image $\sig{in}$, i.e., $\delta = \sig{perturb}(\sig{snn}, \sig{in}, (i,j))$, as the minimum amount of perturbation to be applied on $\sig{in}$ such that the network changes its prediction on the specific pixel. That is, $\delta$ can be characterized using the following equation:
\begin{multline*}
\sig{argmax}(\sig{snn}(\sig{in}, (i,j))) \neq \sig{argmax}(\sig{snn}(\sig{in}+\delta, (i,j)))\\ 
\text{and}\\
\forall \delta': \sig{argmax}(\sig{snn}(\sig{in}, (i,j))) \neq \sig{argmax}(\sig{snn}(\sig{in}+\delta', (i,j)))\\ \rightarrow |\delta|_{L_{\infty}} \leq |\delta'|_{L_{\infty}}
\end{multline*}

In the above formulation, we define the the concept of ``minimum" using $L_{\infty}$ norm; readers can easily change the formulation to alternative $L$-norms formulations such as~$L_{1}$ norm or~$L_{2}$ norm.
Given a fixed perturbation budget $\Delta$ subject to the $L_{\infty}$ norm, we can also compute~$\sig{prm}_{\Delta}(\sig{snn}, \sig{in})$, the \emph{perturbation robustness metric}, to reflect the maximum number of pixels changing from a correct prediction to an incorrect prediction. Let $S := \{(i,j)\;|\;\sig{argmax}(\sig{snn}(\sig{in}, (i,j))) = \sig{gt}(\sig{in}, (i,j))\}$ be the indices of pixels having correct prediction. Then the perturbation robustness metric is computed using the below formula. 
\vspace{-2mm}
\begin{multline*}
\sig{prm}_{\Delta}(\sig{snn}, \sig{in}) := \max_{|\delta|_{L_\infty} \leq \Delta}  \frac{|S \cap  P|}{|S|} \;\;\;\;\;\;\text{with}\\
P := \{(i,j) | \;\sig{argmax}(\sig{snn}(\sig{in}+\delta, (i,j))) \neq \sig{gt}(\sig{in}, (i,j))\}\}
\end{multline*}

\subsection{Principles for safety-aware  metrics}\label{subsec.designing.safety.aware.metrics}

The metrics presented in the previous section are essentially computed by \emph{summing} information collected from  predictions of individual pixels. While for object classification and detection, misclassifying an object can directly be linked to safety issues, the link is weaker in semantic segmentation. That is, inaccurate semantic segmentation does not necessarily cause safety issues. This motivates us to raise the following considerations in building safety-aware metrics by considering the following principles.

\subsubsection{The impact of single pixel class-flipping} We start with the following simplest problem: \emph{if there exists a single pixel that is not classified correctly, what is its impact on safety?} We use Figure~\ref{fig:pinhole.camera} to estimate the impact. Let $f$ being the focal length of the camera, and assume that we use a camera with~$f$ being $28mm$, and let $K=1080$ be the number of pixels in the height. Let~$D=50$ meters  be the specified longest distance where problematic decisions may have a safety impact, e.g., the distance to identify all pedestrians. Let~$d$ be $10$ meters. In other words, an object whose height is~$10$ meters  can completely occupy the height of the image sensing area. Therefore, a single pixel on the image plane, when translating back to the~$50$ meter setup, may occupy around $\frac{10}{1080} = 0.009259$ meter, i.e., around~$0.93$ cm. Analogously, a single pixel can only contribute around~$0.46$ cm at~$25$ meter range. When the safety specification explicitly ignores tiny objects below the dimension mentioned above, we can derive the following statement in layman's words.

\begin{proposal}\label{proposal.size}
	Perturbing a single pixel (or a small cluster of pixels) in the semantic segmentation may not impact safety, provided that conservatively restoring the perturbed pixel to physical objects is relatively small subject to a given safety distance.
\end{proposal}

\subsubsection{The impact of class flipping to adjacent pixels} The second problem is regarding multiple pixels flips to the class adjacent to them. This happens on object borders where the separation between an object (e.g., pedestrian) and the background is hard to justify - thereby inducing labeling uncertainties. Provided that a proper safety buffer~$\eta$ is set in the subsequent computation modules (e.g., modules for generating collision-free trajectory) and the safety buffer~$\eta$ sufficiently covers the potential error, we have the following proposal.

\vspace{1mm}
\begin{proposal}\label{proposal.border}
	For pixels around the borders of an object, flipping their object classes to the outside-border class may not impact safety, provided that sufficiently large safety buffers have been pre-allocated in the subsequent planning and control module.
\end{proposal}

\begin{figure}[t]
	\centering
	\includegraphics[width=0.9\columnwidth]{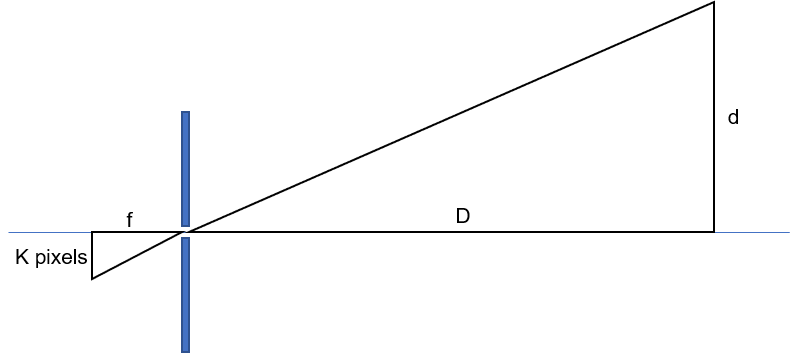}
	
	\caption{Understanding the effect of single pixel decision problem (picture does not draw to according to the actual scale)}
	\label{fig:pinhole.camera}
\end{figure}

\subsubsection{The impact of clustering and relative location}
Finally, we consider the impact of clustering and the relative location of the cluster. We use Figure~\ref{fig:metric.principle} to assist understanding. Given that there exist~$5$ pixels with incorrect classification, our designed metric should reflect that case~(c) is the most safety-critical, followed by (b) and (a). Case~(a) being least critical is explained in previous paragraphs; a larger pixel-cluster implies a  more significant  occupation in the physical space. However, assuming that the image is taken from the front-facing camera, then abnormality  occurred in (c) requires more attention due to its falling inside the planned immediate vehicle trajectory. 

\begin{proposal}\label{proposal.location}
	In addition to clustering effects, within a predefined region in the image that requires additional safety attention, the metric shall weigh more when perturbation or misclassification occurs in that region. 
\end{proposal}

\begin{figure}[t]
	\centering
	\includegraphics[width=\columnwidth]{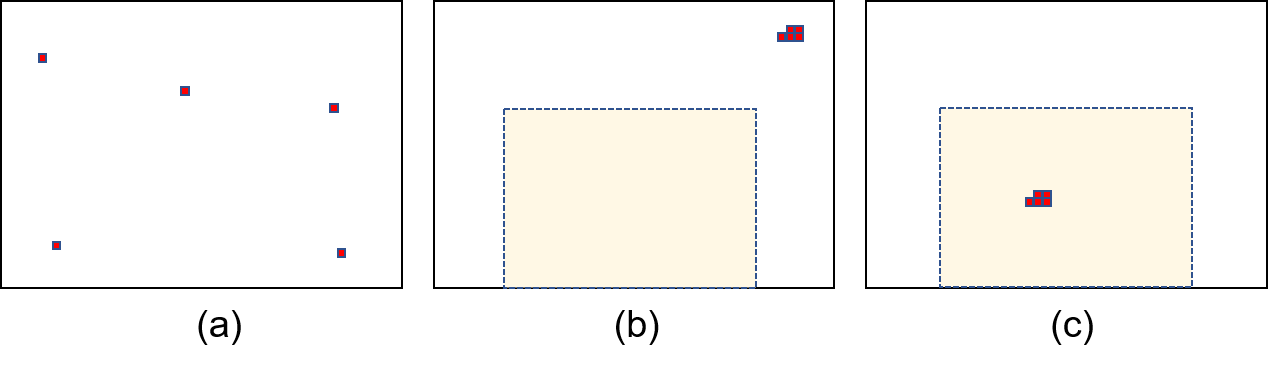}
	\caption{Three figures with each having $5$ pixels of problematic DNN decisions}
	\label{fig:metric.principle}
\end{figure}

\noindent\textbf{(Remark)} Note that in the metric design, one can weigh certain types of class flips to be more severe than others, following the paradigm of weighted confusion. In this paper, we present a generic framework by weighting them equally. Readers are encouraged to adopt the framework by introducing more principles to tailor to their concrete needs.

\subsection{Safety-aware qualitative metric characterization}\label{subsec.safety.aware.qualitative.metric}

\begin{figure}[t]
	\centering
	\includegraphics[width=\columnwidth]{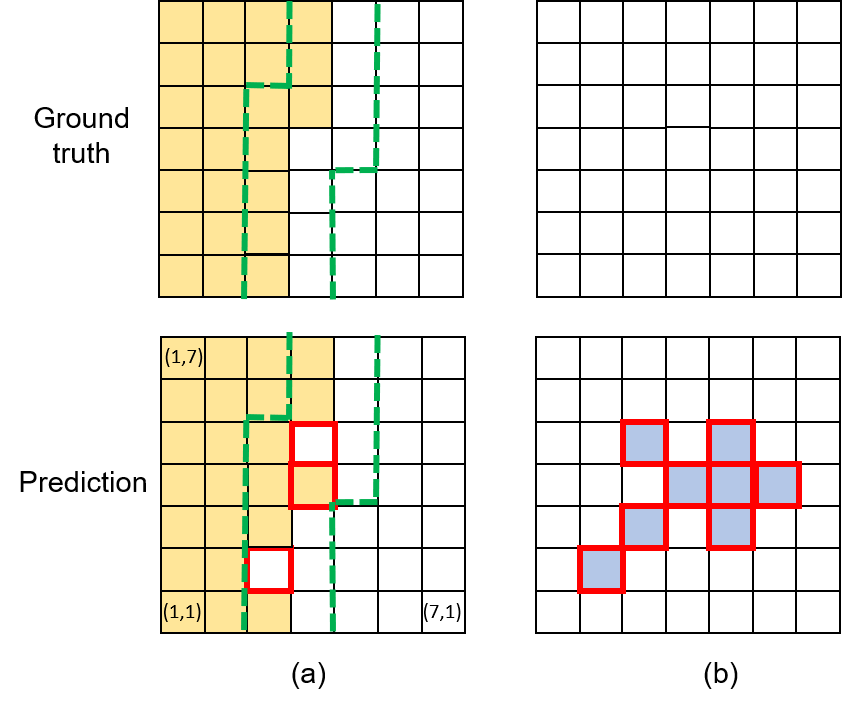}
	\caption{Filters are used to suppress those imperfect edge predictions that do not lead to safety issues; edges in semantic segmentation is surrounded by two green dashed lines~(a).  Applying a $5\times 5$ filter reports ``\sig{unsafe}" for the clustering of prediction errors~(b). The color in each pixel indicates  the corresponding semantic segmentation class. Pixels being wrongly classified are highlighted in red color with bold borders.}
	\label{fig:metric.filter}
	\vspace{-5mm}
\end{figure}

We consider the first type of safety metric being \emph{qualitative}, where $\sig{true} / \sig{safe}$ guarantees the absence of safety issues subject to the principle, while $\sig{false} / \sig{unsafe}$ implies  the result of ground-truth semantic segmentation may imply safety concerns. We use $\sig{c}(i,j)$ to characterize whether the prediction of the pixel indexed at~$(i,j)$ is correct. In other words, 
\vspace{-1mm}
\begin{multline}
\label{eq.correctness.prediction}
\sig{c}(i,j) := 
\sig{argmax}(\sig{snn}(\sig{in}, (i,j))) = \sig{gt}(\sig{in}, (i,j))\\ ? \;\; \sig{true} \;\;:\;\; \sig{false} 
\end{multline}

Subsequently, we follow a three-step process to compute the safety metric. Essentially, the computation  suppresses specific pixel prediction errors (i.e., by swapping $\sig{c}(i,j)$ from $\sig{true}$ to $\sig{false}$) as they do not cause safety concerns; readers may add additional or remove particular filters for their specific use case. 

\vspace{1mm}

\textbf{Step 1: Decide the error treatment on non-critical area.} Regarding errors outside the critical region (e.g., errors appeared on the image border), if the safety principle specifies that these errors never have a safety impact, the filter may completely suppress the prediction error (Proposal~\ref{proposal.location}). We set the predicate $\sig{outside-critical-region}(i,j)$ to $\sig{true}$ iff $(i,j)$ is outside the critical region. 

\vspace{1mm}

\textbf{Step 2: Conservatively neglect errors appearing on edges.} Our second step is to consider the impact of wrong predictions on the edges of the semantic object (Proposal~\ref{proposal.border}). An error-pixel can be \emph{tolerated by edge-imprecision} if (1) in the corresponding semantic segmentation labeling figure, the pixel is on the edges, and (2) the generated class prediction is to the labels on the opposite side of the edge. Altogether we have the following definition. 
\begin{multline}\label{eq.edge.error.no.safety.impact}
\sig{edge-error-neglected}(i,j) := (\sig{c}(i,j) = \sig{false} \; \wedge \\
\sig{belong-edge-in-semantic}(i,j)\; \wedge \;\\  \sig{argmax}(\sig{snn}(\sig{in}, (i,j))) \in  \sig{edge-correct-labels}(i,j)) \\
\; ? \; \sig{true} \; : \; \sig{false} 
\end{multline}

In Eq.~\ref{eq.edge.error.no.safety.impact}, $\sig{belong-edge-in-semantic}(i,j)$ is a predicate checking in the ground-truth semantic segmentation, whether pixel $(i,j)$ is located at an edge. The predicate $\sig{edge-correct-labels}(i,j)$ returns classes of the center pixel $(i,j)$ as well as its eight neighboring pixels within a $3 \times 3$ patch.
An simplified example can be found in Figure~\ref{fig:metric.filter}(a), where based on the ground-truth, pixel indexed~$(3,2)$ is sitting on the edge, thus $\sig{belong-edge-in-semantic}(3,2) = \sig{true}$. By taking the ground truth labels on two sides of the edge, one derives $\sig{edge-correct-labels}(3,2) = \{\sig{orange}, \sig{white}\}$. Therefore,  $\sig{edge-error-neglected}(3,2) = \sig{true}$, meaning that in our safety analysis, the prediction error at location $(3,2)$ can be neglected.

\begin{figure}[t]
	\centering
	\includegraphics[width=\linewidth]{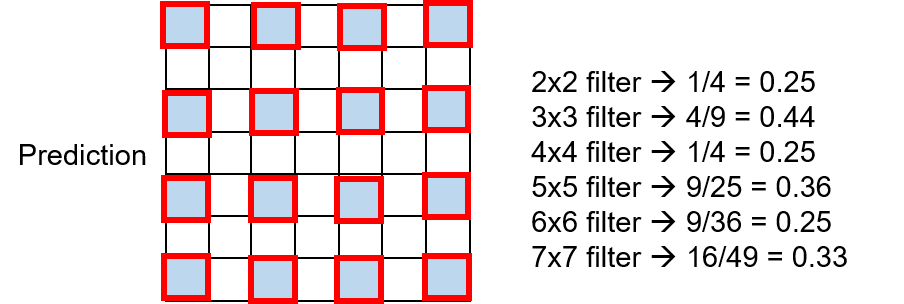}
	\caption{An example where $\alpha = 0.4$ can be ensured in the $2\times 2$ filter but not on the $3\times 3$ filter.}
	\label{fig:error.rate.increase}
\end{figure}

\vspace{1mm}
\textbf{Step 3: Detect dense clustering of incorrect pixel predictions.} Finally, we take the smallest $k_{safe}\times k_{safe}$ filter, where an image of pixel size $k_{safe}\times k_{safe}$, whenever being conservatively projected to the distance~$D$ as demonstrated in Figure~\ref{fig:pinhole.camera}, can lead to safety concerns (Proposal~\ref{proposal.size}). We use the filter to scan through the image, and during scanning,  whenever a large portion  of the $k_{safe}\times k_{safe}$ region demonstrates incorrect predictions with the \emph{error density} larger than the \emph{density threshold}~$\alpha$, we consider this region negatively impact safety, and the resulting metric should report \sig{unsafe}. For example, consider a $5\times5$ filter to be applied on Figure~\ref{fig:metric.filter}(b), provided that $\alpha = 0.2$, then when the filter is placed at the center of the image, the error density has reached $\frac{8}{25} = 0.32$, so the computed safety metric should return \sig{unsafe}.

It is tempting to conclude that it suffices to scan the image using a $k_{safe}\times k_{safe}$ filter. That is, if the filter reports \sig{safe} by proving that every  $k_{safe}\times k_{safe}$ region has an error density below $\alpha$, for any filter of larger size ($k' \times k'$ where $k' > k_{safe}$) the result remains \sig{safe}. Unfortunately, \emph{the above statement does not hold in general}. A counter-example is shown in Figure~\ref{fig:error.rate.increase}, where for the case of $2\times 2$ filter, the maximum error rate equals $\frac{1}{4}$;  in the $3 \times 3$ filter, the maximum error rate has increased to $\frac{4}{9}$. Therefore, if the error density threshold is set to be~$0.4$, the inductive property does not hold.

Based on the counter-example, it is not sufficient to only use the $k_{safe}\times k_{safe}$ filter. Therefore one shall gradually increase the filter size but $k$ can be as large as the width of the image. This makes the computation extremely ineffective as one needs to try all filters with sizes from small to large. In the following, we provide an iterative method to substantially speed up the computation, while the soundness (in terms of reporting ``\sig{safe}") is maintained. 

\begin{enumerate}

	\item Assume that the sum of all error predictions, i.e., $| \{(i,j) \; | \;\sig{c}(i,j) = \sig{false} \}| $ equals~$C_1$. We first find the minimum positive integer~$K_1$ such that the following condition holds.
\begin{equation}\label{eq:minimum.filter.size}
\frac{C_1}{(K_1)^2} < \alpha
\end{equation}

The computed $K_1$ value acts as a conservative estimate - whenever a $k \times k$ filter is used, and if $k \geq K_1$, it is impossible for the error density to be larger or equal to~$\alpha$.  For example, let  $ C_1=10000$. The smallest $K_1$ for $\frac{10000}{(K_1)^2} < 0.5$ equals $142$, meaning that any $k \times k$ grid where $k \geq 142$, it is impossible to be larger than the error density threshold. Therefore, one at most should try filter until size $141 \times 141$.

	\item Apply $(K_1-1) \times (K_1-1)$ filter, and compute the maximum possible errors that can appear in the image, and let this number be $C_2$. Continuing the example, assume that at most $1000$ prediction errors can occur when applying the $141 \times 141$ filter. Therefore, we set $C_2$ to be $1000$. 

\item Now we can again ask a question similar to (1): what is the minimum~$K_2$ such that the following situation occurs.
\vspace{-1mm}
\begin{equation}\label{eq:find.smallest.retangle}
\frac{C_2}{(K_2)^2} < \alpha
\end{equation}

Continue again with the example, the smallest~$K_2$ for ensuring $\frac{1000}{(K_2)^2} < 0.5$ equals~$45$, meaning that applying any $k \times k$ filter  where $k \geq 45$, it is impossible to derive the error density larger than the threshold~$0.5$. In other words, even when we take the maximum number of $1000$ error predictions obtained from the $142\times 142$ filter, align them tightly into a rectangle, under a $45\times 45$ filter, the error density can still only be $\frac{1000}{(45)^2} = 0.494 < 0.5$.

\item Repeat steps (2) and (3) until one of the following situation occurs.

\begin{enumerate}
	\item $k_{safe}$ is reached and the error density remains below~$\alpha$ (\sig{safe}).
	\item Stop at a particular $k \times k$ filter, as one has detected a  case where the for a particular $k \times k$ region, its error density is above~$\alpha$  (\sig{unsafe}).

\end{enumerate}

\end{enumerate}

\noindent Algorithm~\ref{algo:safety.aware.qualitative.metrics}  summarizes the overall qualitative metric computation process, which is simply realizing the steps mentioned above.\footnote{Note that in Algorithm~\ref{algo:safety.aware.qualitative.metrics}, we only construct a squared filter ($k \times k$). The algorithm can be easily extended to other filter shapes such as $k \times 4k$ to incorporate fine-grained considerations such as filtering thin objects such as pedestrians.} The termination of the process relies on an assumption that the filter size~$k$ reduces (with quantity at least~$1$) in every iteration  such that either~$\alpha$ is violated or~$k_{safe}$ will be reached eventually. We prove in the following Lemma~\ref{lemma.termination} that such assumption always holds.

\begin{lem}~\label{lemma.termination} 
The filter size $k$ in Algorithm~\ref{algo:safety.aware.qualitative.metrics} decreases iteratively until $\alpha$ is violated or $k_{safe}$ is reached.
\end{lem}

\begin{proof}
Assume at iteration~$i$, the algorithm has computed the maximum number of erroneous predictions $C_i$ within a $k \times k$ filter, and condition at line~13 does not hold. Given the error density threshold~$\alpha$, at line~16 and~17, we first update the filter size~$k$. Here we use $k_i$ to represent the computed result for~$k$ in the $i$-th round.

Then, we proceed with the next iteration (i.e., $i+1$) of the while loop, and scan through the error image with the filter $k_i \times k_i$ derived from last iteration. This results in an updated error count~$C_{i+1}$ and  two possible cases, namely \sig{unsafe} and \sig{safe}. The former case (i.e. $\frac{C_{i+1}}{(k_i)^2} \geq \alpha$) is trivial as the algorithm can cease and report the result directly, guaranteeing termination. For the latter (i.e. $\frac{C_{i+1}}{(k_i)^2} < \alpha$), we note the following. 
\begin{equation}\label{eq:count.inequality}
C_{i+1} < \alpha \times (k_i)^2
\end{equation}

Then line 16 and 17 essentially calculate the next filter size~$k_{i+1}$ as:
\begin{equation}
k_{i+1} = \lfloor \; \sqrt[]{\frac{C_{i+1}}{\alpha}} \; \rfloor 
\end{equation}

Considering Eq.~\ref{eq:count.inequality} and the floor operation, we have:  
\begin{equation}
k_{i+1} =  \lfloor \sqrt[]{\frac{C_{i+1}}{\alpha}} \rfloor \leq \sqrt[]{\frac{C_{i+1}}{\alpha}} 
<   \sqrt[]{\frac{\alpha \times (k_i)^2}{\alpha}} 
=  k_i
\end{equation}
\end{proof} 
\vspace{-2mm}

\begin{algorithm}[t]
	\SetAlgoLined
\DontPrintSemicolon
    \SetKwInOut{Input}{input}
    \Input{2D map $\sig{c}$ where $\sig{c}(i,j)$ is defined using Eq.~\ref{eq.correctness.prediction}, smallest filter size to have safety impact $k_{safe}$, density threshold $\alpha$, the length of the image~\sig{img-length}  }
     \For{$(i,j)$}{
    	\lIf{$\sig{outside-critical-region}(i,j)$}{$\sig{c}(i,j) := \sig{true}$}
    }
    
    \For{$(i,j)$}{
    	\lIf{$\sig{edge-error-neglected}(i,j)$}{$\sig{c}(i,j) := \sig{true}$}
	}

	\textbf{let} $k := \sig{img-length}, i := 1$\;
	\While{$k \geq k_{safe}$}{
		
		$C_i := 0$
		
		\For{every $k\times k$ grid $G$ in the 2D map $\sig{c}$}
		{$C_i := \max(C_i,  | \{(i,j) \; | \;(i,j) \in G \newline  \wedge \sig{c}(i,j) = \sig{false} \}|)$}
		
		\eIf{$\frac{C_i}{k^2} \geq \alpha$}{
			\Return{\sig{unsafe}}
		}{
		
			\tcc{Find smallest filter that guarantees not to violate~$\alpha$, subject to the maximal error counts.}
			\textbf{let} $K_i := \min \{x\;|\;x \in \mathbb{N} \wedge  \frac{C_i}{x^2} < \alpha\}$\; 
			$k = K_i - 1$ 
			
		}
	}

	\Return{\sig{safe}}\;

	\caption{Computing safety-aware qualitative metrics for semantic segmentation}\label{algo:safety.aware.qualitative.metrics}
\end{algorithm}

\vspace{3mm}

\noindent\textbf{(Safety-aware robustness qualitative metrics)} Here we omit technical formulations, but the above qualitative analysis  can  naturally be extended for robustness considerations. Errors caused by perturbation increase the overall error density. Nevertheless, so long if the error density remains lower than the density threshold regulated by safety principles, one can safely neglect the error caused by perturbation.

\subsection{Moving towards safety-aware quantitative metrics}

\begin{figure}[t]
	\centering
	\includegraphics[width=\columnwidth]{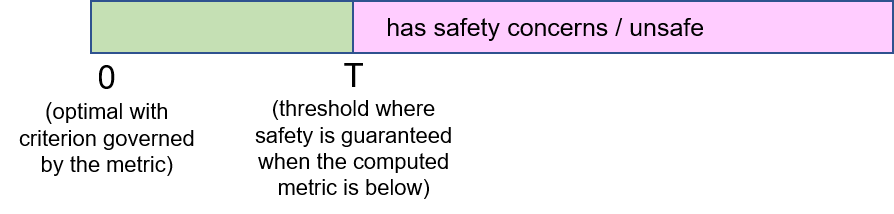}
	\caption{General principle in designing quantitative metrics}
	\label{fig:quantitative.metric}
\end{figure}

In this section, we highlight some simple extensions to transfer our qualitative safety metric to coarse-grained quantitative safety metrics, where the computed numerical value indicates the degree of violation (Figure~\ref{fig:quantitative.metric}). 
\begin{itemize}
	\item \textbf{(Quantify the degree using maximum grid size)} Recall in Algorithm~\ref{algo:safety.aware.qualitative.metrics} that it returns $\sig{unsafe}$ (lines~$13$ and~$14$) when there exists a $k \times k$ grid~$G$ such that the associated errors, dividing by the area size $k^2$ exceeds the density threshold. Therefore, the first method is to simply return the value~$k$ to quantify the degree by associating it with region size.  
	
	\item \textbf{(Quantify the degree using maximum error density)} The second possibility is to return the maximum error density. For this purpose, Algorithm~\ref{algo:safety.aware.qualitative.metrics} needs to be modified such that it does not stop upon encountering $\sig{unsafe}$ situations but remains to bookkeep the error density for every $k \geq k_{safe}$.
\end{itemize}

%% file: sections/4_Evaluation.tex
\section{Evaluation}\label{sec.evaluation}

\begin{figure*}[]
	\centering
	\begin{subfigure}{\textwidth}
		\centering
		\includegraphics[width=0.24\textwidth]{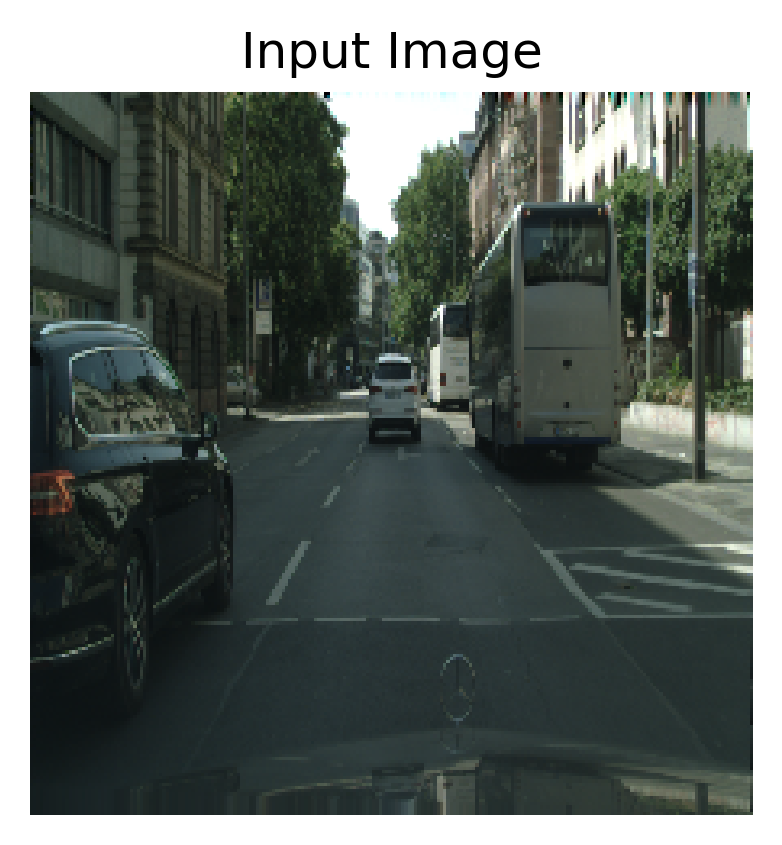}
		\includegraphics[width=0.24\textwidth]{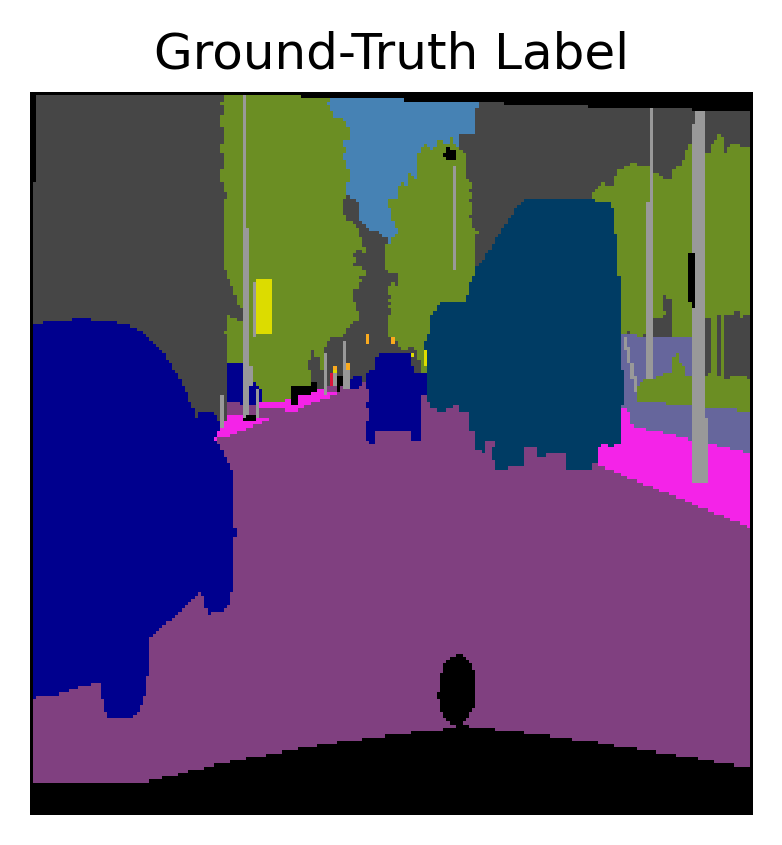}
		\includegraphics[width=0.24\textwidth]{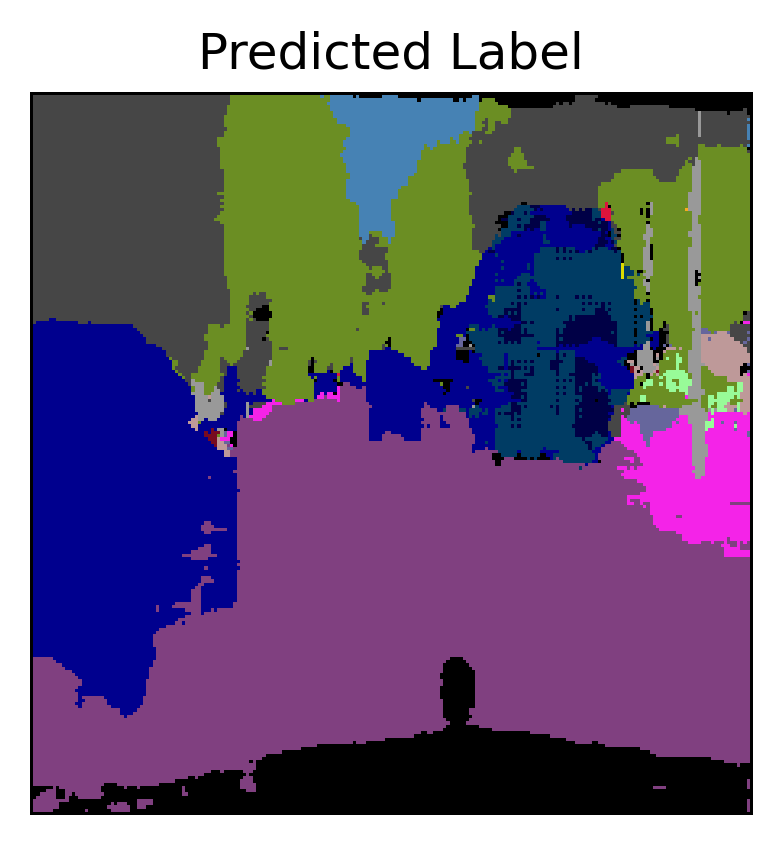}
		\includegraphics[width=0.24\textwidth]{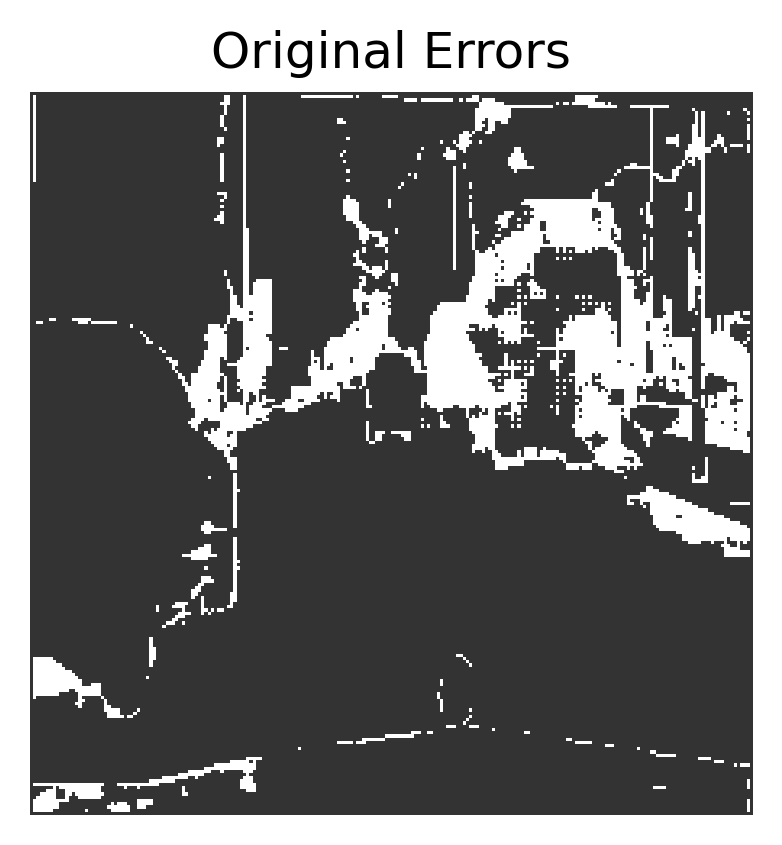} \\
		\includegraphics[width=0.24\textwidth]{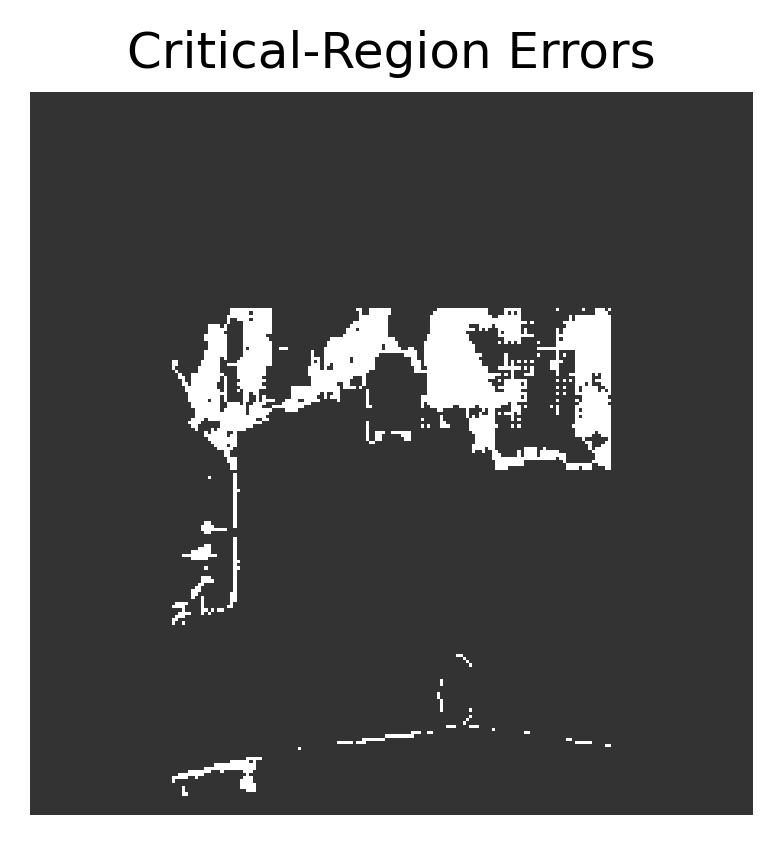}
		\includegraphics[width=0.24\textwidth]{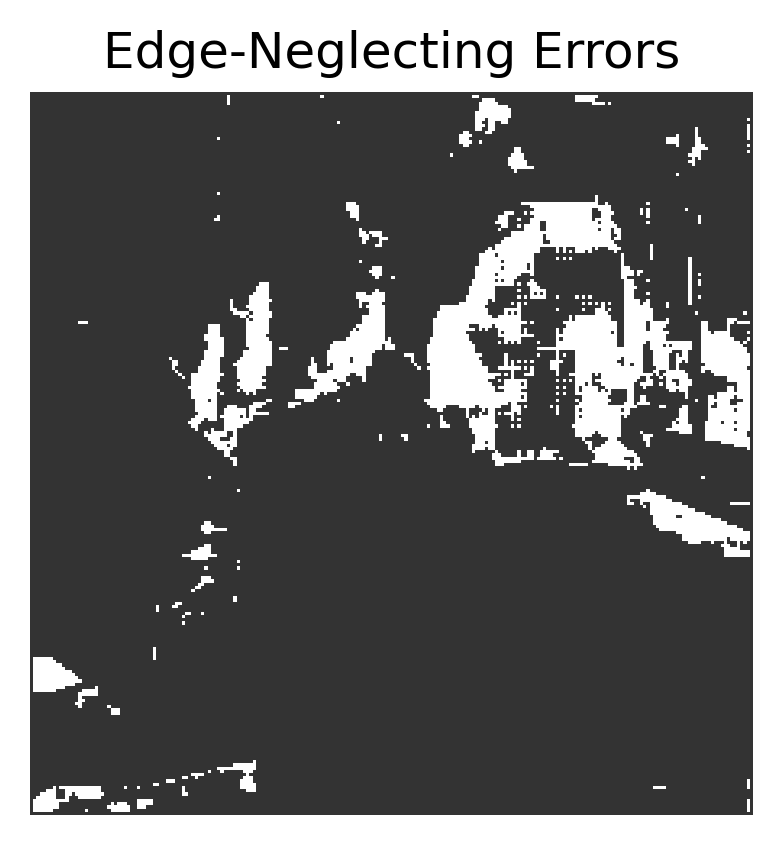}
		\includegraphics[width=0.24\textwidth]{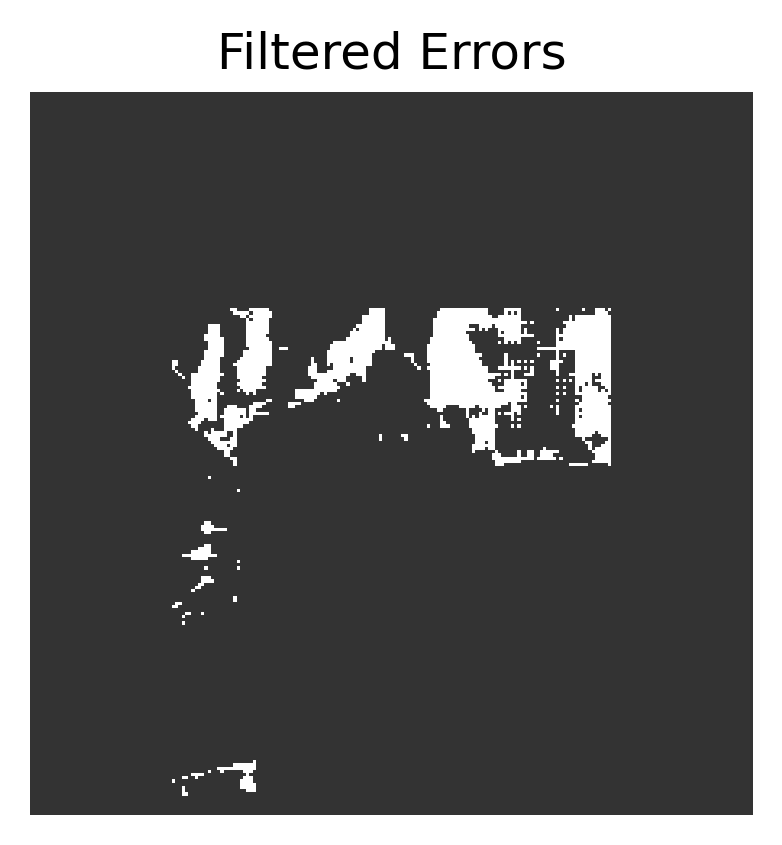}
		\includegraphics[width=0.24\textwidth]{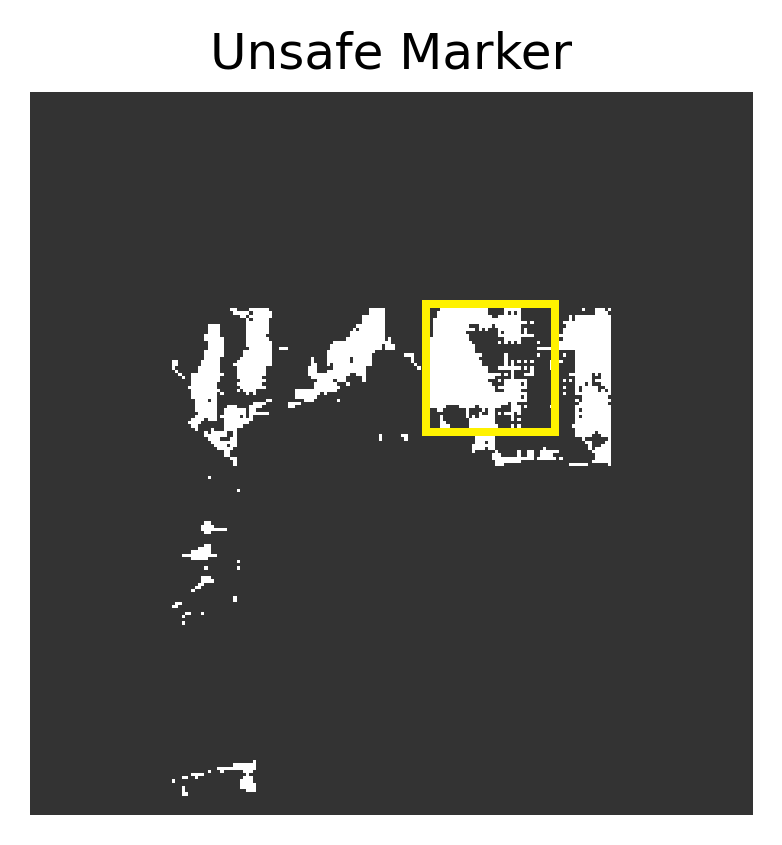}
		\caption{The first sample with a higher accuracy (\sig{pcm} $\approx 0.86$) yet unsafe condition.}
		\label{subfig.first.sample}
	\end{subfigure}
	\begin{subfigure}{\textwidth}
		\centering
		\includegraphics[width=0.24\textwidth]{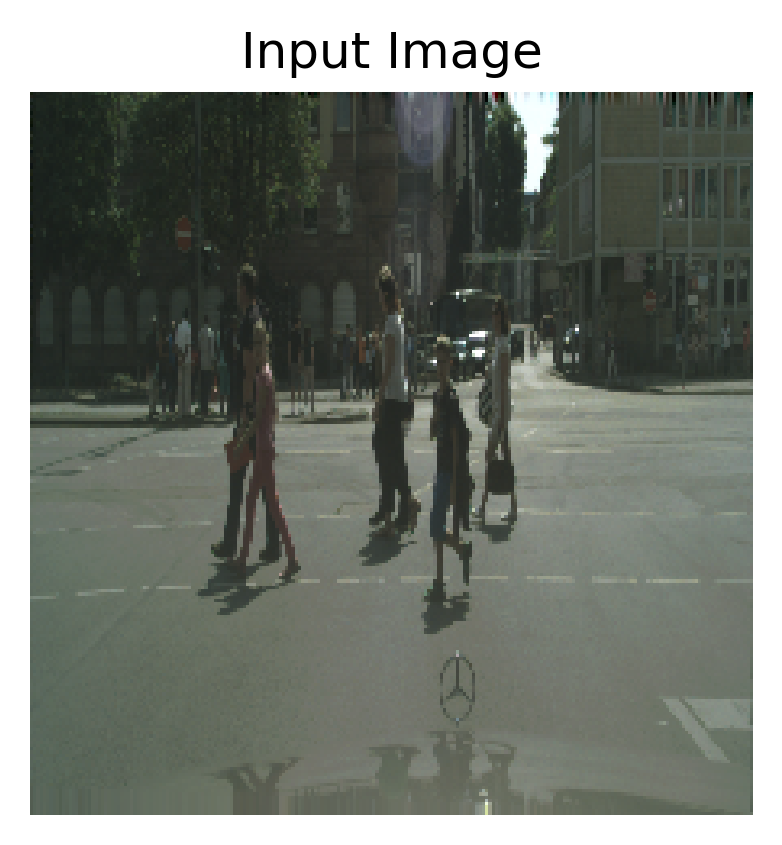}
		\includegraphics[width=0.24\textwidth]{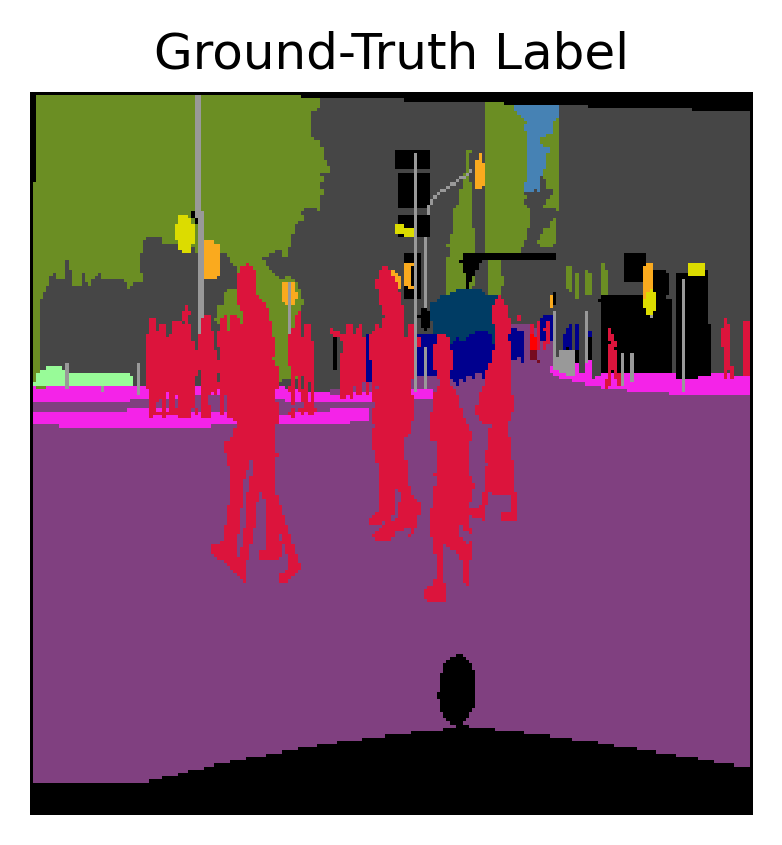}
		\includegraphics[width=0.24\textwidth]{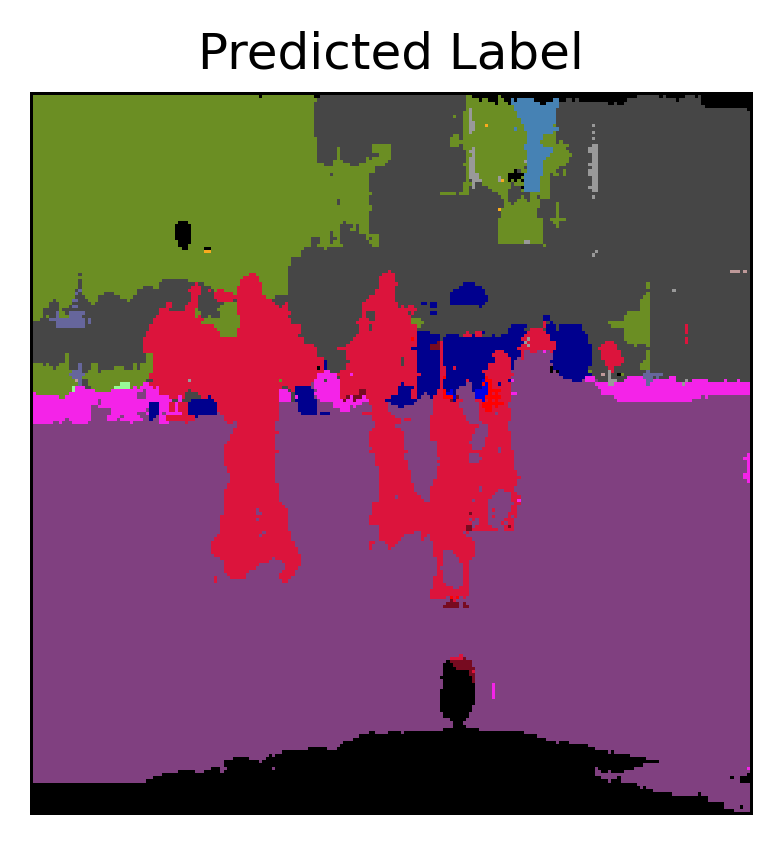}
		\includegraphics[width=0.24\textwidth]{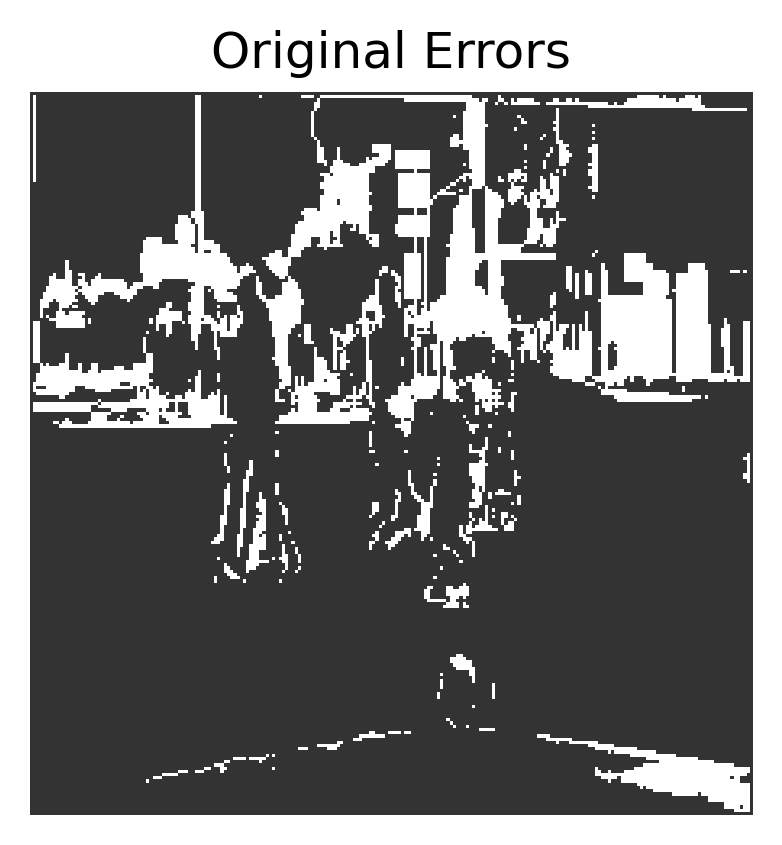} \\
		\includegraphics[width=0.24\textwidth]{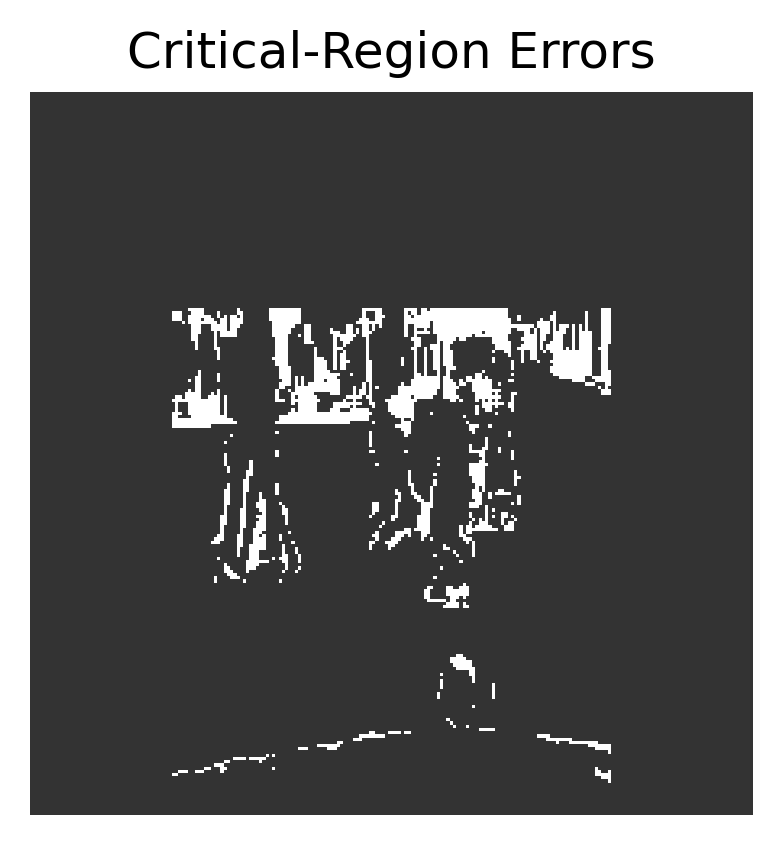}
		\includegraphics[width=0.24\textwidth]{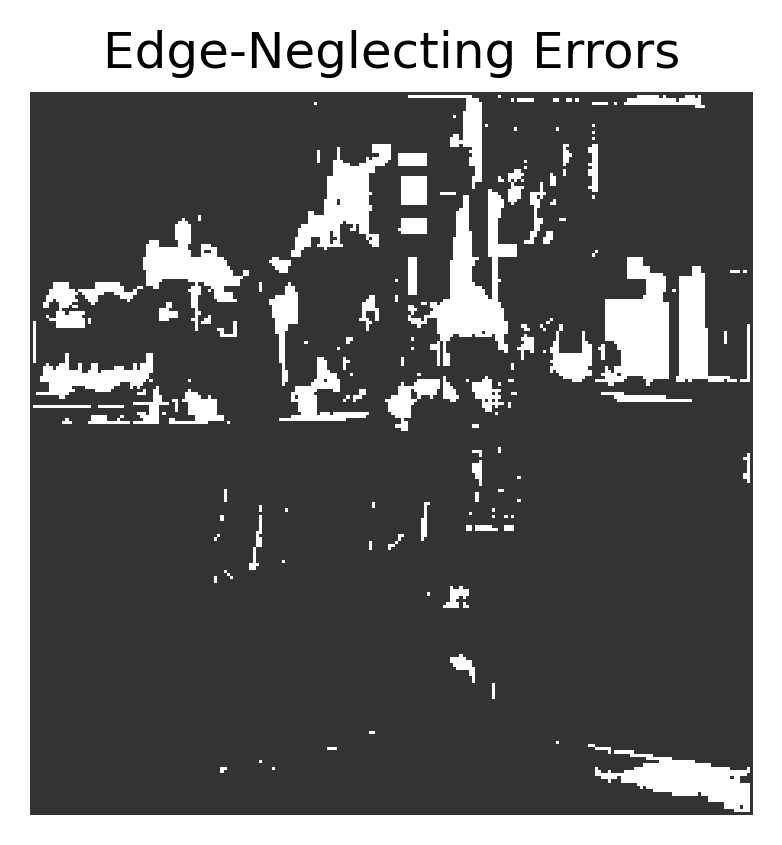}
		\includegraphics[width=0.24\textwidth]{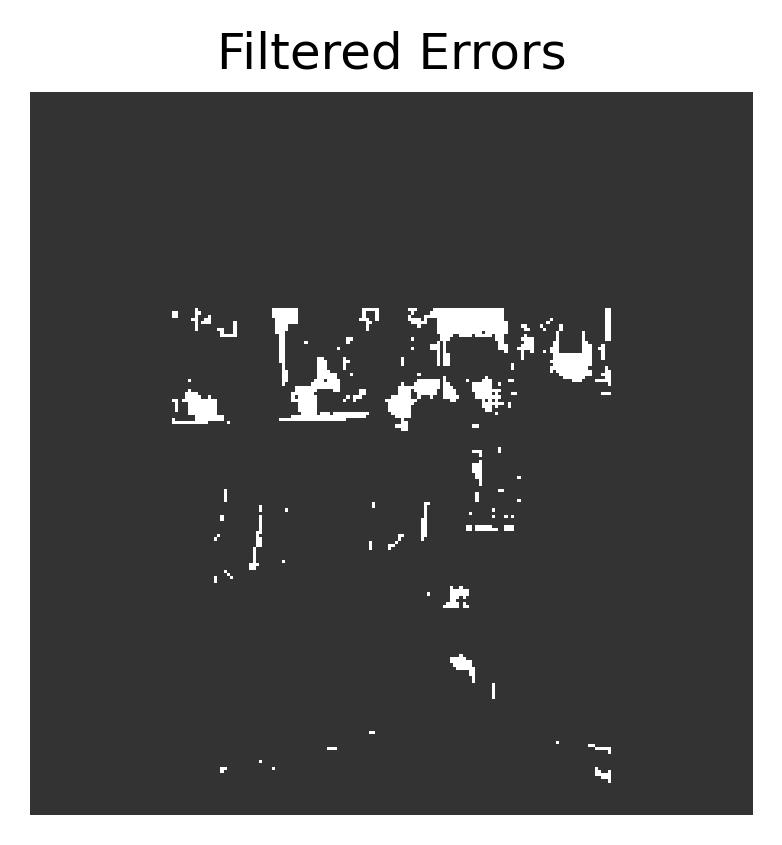}
		\includegraphics[width=0.24\textwidth]{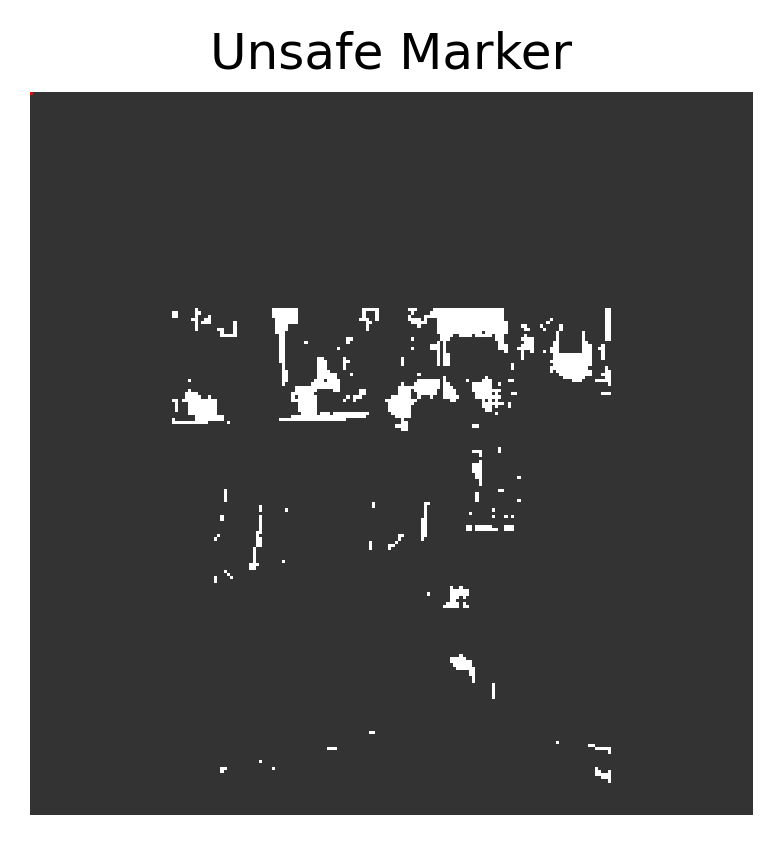}
		\caption{The second sample with a lower accuracy (\sig{pcm} $\approx 0.84$) yet safe condition.}
		\label{subfig.second.sample}
	\end{subfigure}
	\caption{Two samples from the Cityscapes dataset~\cite{Cordts2016Cityscapes} highlighting the need of safety-aware metrics, aside from ordinary accuracy metrics, in autonomous driving scenarios.}
	\label{fig:results}
\end{figure*}

We demonstrate the effect of the proposed safety-aware qualitative metric in this section. The settings of our experiment are first described in Section~\ref{subsec.settings}, trailed by the qualitative results displayed in Section~\ref{subsec.results}.

\subsection{Settings}\label{subsec.settings}

We follow closely the image segmentation tutorial in TensorFlow 2.0\footnote{https://www.tensorflow.org/tutorials/images/segmentation} for the implementation of the vision pipeline. In specific, we employ a U-Net-style  architecture~\cite{ronneberger2015unet} with ImageNet-pretrained MobileNetV2~\cite{howard2018mobilenetv2} as the encoder and the upsampling block of Pix2Pix~\cite{isola2017imagetoimage} as the decoder. We train the model with the ordinary pixel-wise cross-entropy loss, and compare predicted outputs against ground truth labels with the common accuracy metric, which is defined as \sig{pcm} in Section~\ref{subsec.pcm.prm}. Lastly, different from the TensorFlow 2.0 tutorial, we adapt the model to the Berkeley Cityscapes dataset\footnote{The evaluation using the Berkeley Cityscapes dataset in this paper is for knowledge dissemination and scientific publication and is not for commercial use.}~\cite{Cordts2016Cityscapes} as it is more suitable for autonomous driving applications. 

For the safety-aware qualitative metric, we conduct the three-step operation detailed in Section~\ref{subsec.safety.aware.qualitative.metric}. First, we set the critical region at the bottom center of an image, maintaining vertically $0.7$ and horizontally $0.6$ portion of the original image. Second, we neglect the errors occurring on edges within the ground-truth mask due to label ambiguity. Finally, the iterative error density check is triggered with $k_{safe}=20$ and $\alpha=0.5$ to inspect whether the image contains an unsafe condition. Overall, our simple yet general settings establish a solid platform to showcase how conveniently the proposed safety-aware metric can be adopted and how effectively it can work in common semantic segmentation pipelines.

\subsection{Results}\label{subsec.results}

Due to space limits, we present the results using two samples from the dataset as shown in Fig.~\ref{fig:results}. For each sample, we visualize the following eight images: 1) Input image, 2) Ground-truth label, 3) Predicted label, 4) Original errors (by comparing 2 and 3), 5) Critical-region errors (from Section~\ref{subsec.safety.aware.qualitative.metric} \textbf{Step 1}), 6) Edge-neglecting erros (from Section~\ref{subsec.safety.aware.qualitative.metric} \textbf{Step 2}), 7) Filtered errors (by combining~5 and~6) and~8) Unsafe marker (if found from Section~\ref{subsec.safety.aware.qualitative.metric} \textbf{Step~3}).

In terms of the ordinary accuracy metric, the first sample (Fig.~\ref{subfig.first.sample}) has a higher \sig{pcm} $\approx 0.86$ while the second one (Fig.~\ref{subfig.second.sample}) has a lower \sig{pcm} $\approx 0.84$. However, based on our safety inspection, the first sample contains an unsafe circumstance whereas the second one exhibits no safety concerns. Clearly, this case study indicates a reasonable need, if not a strong demand, for safety-aware metrics in safety-critical semantic segmentation use cases such as autonomous driving. Furthermore, our proposed methodology serves reliably as an initiative into this research line.

%% file: sections/5_Conclusion.tex
\section{Concluding Remarks}\label{sec.concluding.remarks}

In this paper, we considered safety-aware metrics for evaluating the prediction quality of semantic segmentation in autonomous driving. The work comes with a practical motivation that using pixel-level classification correctness as a proxy of safety implies the impossibility to provide evidence on safety due to the impossibility of any semantic segmentation DNN generating perfect results. We exemplified with concrete examples how some prediction errors may be neglected subject to principles agreed upfront. The qualitative evaluation demonstrated the validity of our proposal. 

For future work, we are considering integrating these metrics as monitors where the ground truth is pseudo-labeled by other sensor modalities. Another direction is to design concrete loss functions reflecting a safety mindset to fine-tune the semantic segmentation network. Yet another direction is to modify the quantitative metric as a loss function to train the DNN towards safety-by-design.






%% file: draft.bbl
\begin{thebibliography}{10}
	
	\bibitem{UL4600}
	{ANSI/UL} 4600 standard for safety for the evaluation of autonomous products.
	\newblock \url{https://ul.org/UL4600}.
	
	\bibitem{saFAD}
	{ISO/TR 4804:2020 Road vehicles — Safety and cybersecurity for automated
		driving systems — Design, verification and validation}.
	\newblock \url{https://www.iso.org/standard/80363.html}, 2020.
	
	\bibitem{arnab2018robustness}
	A.~Arnab, O.~Miksik, and P.~H.~S. Torr.
	\newblock On the robustness of semantic segmentation models to adversarial
	attacks.
	\newblock {\em PAMI}, 42(12):3040--3053, 2020.
	
	\bibitem{badrinarayanan2017segnet}
	V.~Badrinarayanan, A.~Kendall, and R.~Cipolla.
	\newblock {SegNet}: {A} deep convolutional encoder-decoder architecture for
	image segmentation.
	\newblock {\em PAMI}, 39(12):2481--2495, 2017.
	
	\bibitem{chen2017deeplab}
	L.-C. Chen, G.~Papandreou, I.~Kokkinos, K.~Murphy, and A.~L. Yuille.
	\newblock Deeplab: Semantic image segmentation with deep convolutional nets,
	atrous convolution, and fully connected crfs.
	\newblock {\em PAMI}, 40(4):834--848, 2017.
	
	\bibitem{Cordts2016Cityscapes}
	M.~Cordts, M.~Omran, S.~Ramos, T.~Rehfeld, M.~Enzweiler, R.~Benenson,
	U.~Franke, S.~Roth, and B.~Schiele.
	\newblock The cityscapes dataset for semantic urban scene understanding.
	\newblock {\em CoRR}, abs/1604.01685, 2016.
	
	\bibitem{goodfellow2015explaining}
	I.~J. Goodfellow, J.~Shlens, and C.~Szegedy.
	\newblock Explaining and harnessing adversarial examples.
	\newblock In {\em ICLR}, 2015.
	
	\bibitem{huang2020survey}
	X.~Huang, D.~Kroening, W.~Ruan, J.~Sharp, Y.~Sun, E.~Thamo, M.~Wu, and X.~Yi.
	\newblock A survey of safety and trustworthiness of deep neural networks:
	Verification, testing, adversarial attack and defence, and interpretability.
	\newblock {\em Comput. Sci. Rev.}, 37:100270, 2020.
	
	\bibitem{isola2017imagetoimage}
	P.~Isola, J.-Y. Zhu, T.~Zhou, and A.~A. Efros.
	\newblock Image-to-image translation with conditional adversarial networks.
	\newblock {\em CVPR}, 2017.
	
	\bibitem{jia2021framework}
	Y.~Jia, T.~Lawton, J.~McDermid, E.~Rojas, and I.~Habli.
	\newblock A framework for assurance of medication safety using machine
	learning.
	\newblock {\em CoRR}, abs/2101.05620, 2021.
	
	\bibitem{klingner2020improved}
	M.~Klingner, A.~B{\"a}r, and T.~Fingscheidt.
	\newblock Improved noise and attack robustness for semantic segmentation by
	using multi-task training with self-supervised depth estimation.
	\newblock {\em CVPR Workshop}, pages 1299--1309, 2020.
	
	\bibitem{lin2017focal}
	T.-Y. Lin, P.~Goyal, R.~Girshick, K.~He, and P.~Doll{\'a}r.
	\newblock Focal loss for dense object detection.
	\newblock In {\em ICCV}, pages 2980--2988. IEEE, 2017.
	
	\bibitem{liu2020importance}
	X.~Liu, Y.~Han, S.~Bai, Y.~Ge, T.~Wang, X.~Han, S.~Li, J.~You, and J.~Lu.
	\newblock Importance-aware semantic segmentation in self-driving with discrete
	wasserstein training.
	\newblock In {\em AAAI}, volume~34, pages 11629--11636, 2020.
	
	\bibitem{madry2019deep}
	A.~Madry, A.~Makelov, L.~Schmidt, D.~Tsipras, and A.~Vladu.
	\newblock Towards deep learning models resistant to adversarial attacks.
	\newblock In {\em ICLR}, 2018.
	
	\bibitem{paszke2016enet}
	A.~Paszke, A.~Chaurasia, S.~Kim, and E.~Culurciello.
	\newblock Enet: {A} deep neural network architecture for real-time semantic
	segmentation.
	\newblock {\em CoRR}, abs/1606.02147, 2016.
	
	\bibitem{ronneberger2015u}
	O.~Ronneberger, P.~Fischer, and T.~Brox.
	\newblock U-net: Convolutional networks for biomedical image segmentation.
	\newblock In {\em MICCAI}, pages 234--241. Springer, 2015.
	
	\bibitem{ronneberger2015unet}
	O.~Ronneberger, P.~Fischer, and T.~Brox.
	\newblock U-net: Convolutional networks for biomedical image segmentation.
	\newblock In N.~Navab, J.~Hornegger, W.~M. Wells, and A.~F. Frangi, editors,
	{\em MICCAI}, pages 234--241, 2015.
	
	\bibitem{howard2018mobilenetv2}
	M.~Sandler, A.~Howard, M.~Zhu, A.~Zhmoginov, and L.-C. Chen.
	\newblock Mobilenetv2: Inverted residuals and linear bottlenecks.
	\newblock In {\em CVPR}, 2018.
	
	\bibitem{long2015fully}
	E.~Shelhamer, J.~Long, and T.~Darrell.
	\newblock Fully convolutional networks for semantic segmentation.
	\newblock {\em PAMI}, 39(4):640--651, 2017.
	
	\bibitem{siam2018comparative}
	M.~Siam, M.~Gamal, M.~Abdel-Razek, S.~Yogamani, M.~Jagersand, and H.~Zhang.
	\newblock A comparative study of real-time semantic segmentation for autonomous
	driving.
	\newblock In {\em CPVR Workshop}, pages 587--597, 2018.
	
	\bibitem{szegedy2014intriguing}
	C.~Szegedy, W.~Zaremba, I.~Sutskever, J.~Bruna, D.~Erhan, I.~J. Goodfellow, and
	R.~Fergus.
	\newblock Intriguing properties of neural networks.
	\newblock In {\em ICLR}, 2014.
	
	\bibitem{xu2020quantification}
	P.~Xu, W.~Ruan, and X.~Huang.
	\newblock Towards the quantification of safety risks in deep neural networks.
	\newblock {\em CoRR}, abs/2009.06114, 2020.
	
	\bibitem{xu2020dynamic}
	X.~Xu, H.~Zhao, and J.~Jia.
	\newblock Dynamic divide-and-conquer adversarial training for robust semantic
	segmentation.
	\newblock {\em CoRR}, abs/2003.06555, 2020.
	
	\bibitem{zhao2020safety}
	X.~Zhao, A.~Banks, J.~Sharp, V.~Robu, D.~Flynn, M.~Fisher, and X.~Huang.
	\newblock A safety framework for critical systems utilising deep neural
	networks.
	\newblock In {\em SAFECOMP}, pages 244--259. Springer, 2020.
	
\end{thebibliography}
